\documentclass[conference]{IEEEtran}
\makeatletter
\def\ps@headings{%
\def\@oddhead{\mbox{}\scriptsize\rightmark \hfil \thepage}%
\def\@evenhead{\scriptsize\thepage \hfil \leftmark\mbox{}}%
\def\@oddfoot{}%
\def\@evenfoot{}}
\makeatother 

\usepackage{cite}
\usepackage{subfigure}
\usepackage{url}
\usepackage{algorithm}
\usepackage{array}
\usepackage{amsmath,amsfonts,amssymb,graphicx}
\usepackage{bm}
\usepackage{mathrsfs}
\usepackage{dsfont}
\usepackage{multicol}
\usepackage[end]{algpseudocode}
\usepackage{multirow}

\newcommand{\commnt}[1] {$//$ \textsc{#1} }
\newtheorem{theorem}{\textbf{Theorem}}
\newtheorem{lemma}{\textbf{Lemma}}
\newtheorem{corollary}{\textit{Corollary}}
\newtheorem{fact}{\textbf{Fact}}
\newtheorem{proof}{\textbf{Proof}}

\begin{document}
%
\title{Efficient Online Learning for Opportunistic Spectrum Access}

\author{\IEEEauthorblockN{Wenhan Dai$^\dag$, Yi Gai$^\sharp$ and Bhaskar~Krishnamachari$^\sharp$}
\IEEEauthorblockA{$^\dag${Massachusetts Institute of Technology, Cambridge, MA, USA}\\ $^\sharp$University of Southern California, Los Angeles, CA, USA\\
Email: whdai@mit.edu, \{ygai, bkrishna\}@usc.edu}}


\maketitle

\begin{abstract}
The problem of opportunistic spectrum access in cognitive radio
networks has been recently formulated as a non-Bayesian restless
multi-armed bandit problem. In this problem, there are $N$ arms
(corresponding to channels) and one player (corresponding to a
secondary user). The state of each arm evolves as a finite-state
Markov chain with unknown parameters. At each time slot, the player
can select $K < N$ arms to play and receives state-dependent rewards
(corresponding to the throughput obtained given the activity of
primary users). The objective is to maximize the expected total
rewards (i.e., total throughput) obtained over multiple plays. The
performance of an algorithm for such a multi-armed bandit problem is
measured in terms of regret, defined as the difference in expected
reward compared to a model-aware genie who always plays the best $K$
arms. In this paper, we propose a new continuous exploration and
exploitation (CEE) algorithm for this problem. When no information
is available about the dynamics of the arms, CEE is the first
algorithm to guarantee near-logarithmic regret uniformly over time.
When some bounds corresponding to the stationary state distributions
and the state-dependent rewards are known, we show that CEE can be
easily modified to achieve logarithmic regret over time. In
contrast, prior algorithms require additional information concerning
bounds on the second eigenvalues of the transition matrices in order
to guarantee logarithmic regret. Finally, we show through numerical
simulations that CEE is more efficient than prior algorithms.
\end{abstract}

\section{Introduction}
Multi-arm bandit (MAB) problems are widely used to make optimal
decisions in dynamic environments. In the classic MAB problem, there
are $N$ independent arms and one player. At every time slot, the
player selects $K (\geq 1)$ arms to sense and receives a certain
amount of rewards. In the classic non-Bayesian formulation, the
reward of each arm evolves in i.i.d. over time and is unknown to the
player. The player seeks to design a policy which can maximize the
expected total reward.

One interesting variant of multi-armed bandits is the restless
multi-arm bandit problem (RMAB). In this case, all the arms, whether
selected (activated) or not, evolve as a Markov chain at every time
slot. When one arm is played, its transition matrix may be different
from that when it is not played. Even if the player knows the
parameters of the model, which can be referred to as the Bayesian
RMAB since the beliefs on each arm can be updated at each time based
on the observations in this case, the design of the optimal policy
turns to be a PSPACE hard optimization
problem~\cite{Papadimitriou}.

In this paper, we consider the more challenging non-Bayesian RMAB
problems, in which parameters of the model are unknown to the
player. The objective is to minimize \emph{regret}, defined as the
gap between the expected reward that can be achieved by a suitably
defined genie that knows the parameters and that obtained by the
given policy. As stated before, finding the optimal policy, which is
in general non-stationary, is P-SPACE hard even if the parameters
are known. So we use instead a weaker notion of regret, where the
genie always selects the $K$ most rewarding arms that have highest
stationary rewards when activated.

We propose a sample mean-based index policy without information
about the system. We prove that this algorithm achieves regret
arbitrarily close to logarithmic uniformly over time horizon.
Specifically, the regret can be bound by
$Z_1G(n)\ln{n}+Z_2\ln{n}+Z_3G(n)+Z_4$, where $n$ is time, $Z_i,
i=1,2,3,4$ are constants and $G(n)$ can be any divergent
non-decreasing sequence of positive integers. Since the growth speed
of $G(n)$ can be arbitrarily slowly, the regret of our algorithm is
nearly logarithmic with time. The significance of such a sub-linear
time regret bound is that the time-averaged regret tends to zero (or
possibly even negative since the genie we compare with is not using
a globally optimal policy), implying the time-averaged rewards of
the policy will approach or even possibly exceed those obtained by
the stationary policy adopted by the model-aware genie.

If the some bounds corresponding to the stationary state
distributions and the state-dependent rewards are known, we show
that the algorithm can be easily modified and achieves logarithmic
regret over time. Compared to prior
work~\cite{Haoyang:2011}~\cite{Haoyang:2011_1}~\cite{TekinLiu}, our
algorithm requires the least information about the system; in
particular, we do not require to know the second largest eigenvalue
of transition matrix or multiplicative symmetrization matrix.
Moreover, our simulation results show that our algorithm obtains the
lowest regret compared to previously proposed algorithms when the
parameters just satisfy the theoretical boundaries.

Research in restless multi-arm bandit problems has a lot of
applications. For instance, it has been applied to dynamic spectrum
sensing for opportunistic spectrum access in cognitive radio
networks, where a secondary user must select $K$ of $N$ channels to
sense at each time to maximize its expected reward from transmission
opportunities. If the primary user occupancy on each channel is
modeled as a Markov chain with unknown parameters, then we obtain an
RMAB problem. We conduct our simulation-based evaluations in the
context of this particular problem of opportunistic spectrum access.

The remainder of this paper is organized as follows: in Section
\ref{sec:Rel}, we briefly review the related work on MAB problems.
In Section \ref{sec:Formulation}, we formulate the general RMAB
problem. In Section \ref{sec:Sensing} and Section
\ref{sec:Sensing_K}, we introduce a sample mean based policy and
provide a proof for the regret upper bound separately for single and
multiple channel selection cases. In Section \ref{sec:NR}, we
evaluate our algorithm and compare it via simulations with the RCA
algorithm proposed in~\cite{TekinLiu} and the RUCB proposed
in~\cite{Haoyang:2011} for the problem of opportunistic spectrum
access. We conclude the paper in Section \ref{sec:conclusion}.

\section{Related Work}
\label{sec:Rel}

In 1985, Lai and Robbins proved that the minimum regret grows with
time in a logarithmic order~\cite{Lai:Robbins}. They also proposed
the first policy that achieved the optimal logarithmic regret for
multi-armed bandit problems in which the rewards are i.i.d. over
time. Their policy only achieves the optimal regret asymptotically.
Anantharam \emph{et al.} extended this result to multiple
simultaneous arm plays, as well as single-parameter Markovian rested
rewards~\cite{Anantharam:1987}. Auer \emph{et al.} developed UCB1
policy in 2002, applying to i.i.d. reward distributions with finite
support, achieving logarithmic regret over time, rather than only
asymptotically in time. Their policy is based on the sample mean of the
observed data, and has a rather simple index selection method.

One important variant of classic multi-armed bandit problem is the
Bayesian MAB. In this case, \emph{a priori} probabilistic knowledge
about the problem and system is required. Gittins and Jones
presented a simple approach for the rested bandit problem, in which one
arm is activated at each time and only the activated arm changes
state as a known Markov process~\cite{Gittins:1972}. The optimal
policy is to play the arm with highest Gittins' index. The
\emph{restless bandit problem} was posed by Whittle in
1988~\cite{Whittle}, in which all the arms can change state. The
optimal solution for this problem has been shown to be PSPACE-hard
by Papadimitriou and Tsitsiklis~\cite{Papadimitriou}.  Whittle
proposed an index policy which is optimal under certain
conditions~\cite{Weber:1990}. This policy can offer near-optimal
performance numerically, however, its existence and optimality are
not guaranteed. The restless bandit problem has no general solution
though it may be solved in special cases. For instance, when each
channel is modeled as identical two-state Markov chain, the myopic
policy is proved to be optimal if the channel number is no more than 3
or is positively
correlated~\cite{zhao:krishnamachari:twc2008}~\cite{ahmad:liu:it2009}.

There have been a few recent attempts to solve the restless multi-arm
bandit problem under unknown models. In~\cite{TekinLiu}, Tekin and Liu use a
weaker definition of regret and propose a policy (RCA) that achieves
logarithmic regret when certain knowledge about the system is known.
However, the algorithm only exploits part of observing data and
leaves space to improve performances. In~\cite{Haoyang:2011},
Haoyang Liu \emph{et al.} proposed a policy, referred to as RUCB,
achieving a logarithmic regret over time when certain system
parameters are known. The regret they adopt is the same as
in~\cite{TekinLiu}. They also extend the RUCB
policy to achieve a near-logarithmic regret over time when no
knowledge about the system is available. Conclusions on multi-arm
selections are given in~\cite{Haoyang:2011_1}. However, they only
give the upper bound of regret at the end of a certain time point
referred as \emph{epoch}. When no \emph{a priori} information
about the system is known, their analysis of regret gives the upper
bound over time only asymptotically, not uniformly.

 In our previous work~\cite{W.Dai:2011}, we adopted a stronger definition of regret, which is defined as the reward loss with the optimal policy. Our policy achieve a near-logarithmic regret without \emph{a prior} of the system. It applies to special cases of the RMAB, in particular the same scenario as in~\cite{zhao:krishnamachari:twc2008} and ~\cite{ahmad:liu:it2009}.

\section{Problem Formulation}
\label{sec:Formulation} We consider a time-slotted system with one
player and $N$ independent arms. At each time slot, the player
selects (activates) $K(<N)$ arms and gets a certain amount of rewards
according to the current state of the arm. Each arm is modeled as a
discrete-time, irreducible and aperiodic Markov chain with finite
state space. We assume the arms are independent. Generally, the
transition matrices in the activated model and the passive model are
not necessarily identical. The player can only see the state of the
sensed arm and does not know the transitions of the arms. The player
aims to maximize its expected total reward (throughput) over some
time horizon by choosing judiciously a sensing policy $\phi$ that
governs the channel selection in each slot. Here, a policy is an
algorithm that specifies arm selection based on observation
history.

Let $S^i$ denote the state space of arm $i$. Denote $r_x^i$ the
reward obtained from state $x$ of arm $i$, $x\in S^i$. Without loss
of generality, we assume $r_x^i\leq 1, \forall x\in S^i, \forall i$.
Let $P_j$ denote the active transition matrix of arm $j$ and $Q_j$
denote the passive transition  matrix. Let
$\mathbf{\pi}^i=\{\pi_x^i, x\in S^i\}$ denote the stationary
distribution of arm $i$ in the active model, where $\pi_x^i$ is the
stationary probability of arm $i$ being in state $x$ (under $P_i$).
The stationary mean reward of arm $i$, denoted by $\mu^i$, is the
expected reward of arm $i$ under its stationary distribution:

\begin{equation}
\mu^i = \sum_{x \in S^i}r_x^i\pi_x^i
\end{equation}

Consider the permutation of $\{1,\cdots,N\}$ denoted as $\sigma$,
such that $\mu^{\sigma(1)} > \mu^{\sigma(2)} > \mu^{\sigma(3)} >
\cdots \mu^{\sigma(N)}.$ We are interested in designing policies
that perform well with respect to $regret$, which is defined as the
difference between the expected reward that is obtained by using the
policy selecting $K$ best arms and that obtained by the given
policy. The best arm obtains the highest stationary mean reward.

Let $Y^{\Phi}(t)$ denote the reward obtained at time $t$ with policy
$\Phi$. The total reward achieved by policy $\Phi$ is given by

\begin{equation}
R^{\Phi}(t) = \sum_{j=1}^{t}Y^{\Phi}(t)
\end{equation}

and the regret $r^\Phi(t)$ achieved by policy $\Phi$ is given by

\begin{equation}\label{eqn:regret}
r^\Phi(t) = t\sum_{j=1}^{K}\mu^{\sigma(j)} - \mathbb{E}(R^{\Phi}(t))
\end{equation}

The objective is to minimize the growth rate of the regret.

\section{Analysis for Single Arm Selection}
\label{sec:Sensing} In this section, we focus on the situation when
$K = 1$. In this case, the player selects one arm each time. We
first show an algorithm called \emph{Continuous Exploration and
Exploitation} (CEE) and then prove that our algorithm achieves a
near-logarithmic regret with time.

\subsection{The CEE Algorithm for non-Bayesian RMAB}

Our CEE algorithm (see Algorithm 1) works as follows. We first
process the initialization by selecting each arm for certain time
slots (we call these time slots \emph{step}), then iterate the arm selection by
searching the index that maximizes the equation shown in
line~\ref{line:9} in Algorithm 1 and operating this arm for one
\emph{step}. A key issue is how long to operate each arm at each
step. It turns out from the analysis we present in the next
subsection that it is desirable to slowly increase the duration of
each step using any (arbitrarily slowly) divergent non-decreasing
sequence of positive integers $\{B_i\}_{i=1}^\infty$.

A list of notations is summarized as follows:
\begin{itemize}
\item n:\;time.
\item $B_i$:\;duration of $i_{th}$ step.
\item $\hat{A}_i(i_j)$:\; sample mean of the $i_{j\;th}$ step arm $i$ being selected.
\item $\hat{X}_j$:\;sum of sample mean in all the steps arm $i$ being selected.
\end{itemize}

\begin{algorithm} [ht]
\caption{Continuous Exploration and Exploitation (CEE): Single Arm
Selection} \label{alg:sensing1}
\begin{algorithmic}[1]
\State \commnt{Initialization}

\State Play arm $i$ for $B_i$ time slots, denote $\hat{A}_i(1)$ as
the sample mean of these $B_i$ rewards, $i=1,2,\cdots,N$ \State
$\hat{X}_i = \hat{A}_i(1)$, $i=1,2,\cdots,N$ \State $n =
\sum_{i=1}^{N} B_i$ \State $i=N+1$, $i_j=1$, $j=1,2,\cdots,N$

\State \commnt{Main loop}

\While {1}
    \State Find $j$ such that
    $j = \arg\max\frac{\hat{X}_j}{i_j}+\sqrt{\frac{L\ln{n}}{i_j}}$(L can be any constant greater than 2) \label{line:9}
    \State $i_j = i_j+1$
    \State Play arm $j$ for $B_i$ slots, let $\hat{A}_{j}(i_j)$ record the sample mean of these $B_i$ rewards
    \State $\hat{X}_j = \hat{X}_j + \hat{A}_{j}(i_j)$
    \State $i = i+1$
    \State $n = n + B_i$;
\EndWhile
\end{algorithmic}
\end{algorithm}

\subsection{Regret Analysis}
We first define the discrete function $G(n)$, which represents the
value of $B_i$, at the $n^{th}$ time step in Algorithm 1:
\begin{equation}\label{eq:G(n)}
G(n) = \mathop{\min}_{I} B_I \ s.t.  \mathop{\sum}_{i=1}^{I}B_i\geq
n
\end{equation}

Since $B_i \geq 1$, it is obvious that $G(n) \leq B_n, \forall n$.
Note that
 since $B_i$ can be any arbitrarily slow non-decreasing diverging sequence, $G(n)$ can also grow arbitrarily slowly.

In this subsection, we show that the regret achieved by our
algorithm has a near-logarithmic order. This is given in the
following Theorem \ref{theoremRA:K=1}.

\begin{theorem}\label{theoremRA:K=1}
Assume all arms are modeled as finite state, irreducible, aperiodic
and reversible Markov chains. All the states (rewards) are positive.
The expected regret with Algorithm 1 after $n$ time slots is at most
$Z_1 G(n)\ln{n} + Z_2 \ln{n} + Z_3 G(n) + Z_4$, where
$Z_1,Z_2,Z_3,Z_4$ are constants only related to $P_i, i=1,2,\cdots,
N$, explicit expressions are at the end of proof for Theorem
\ref{theoremRA:K=1}.

\end{theorem}

The proof of Theorem \ref{theoremRA:K=1} uses the following fact and
two lemmas that we present next.

\fact \label{fact:chernoff} (Chernoff-Hoeffding bound) Let
$X_1,\cdots,X_n$ be random variables with common range $[0,1]$ and
such that $\mathbb{E}[X_t|X_1,\cdots,X_{t-1}]=\mu$. Let $S_n =
X_1+\cdots+X_n$. Then for all a $\geq 0$
\begin{equation}
\mathbb{P}\{S_n \geq n\mu + a\}\leq e^{-2a^2/n}; \mathbb{P}\{S_n
\leq n\mu - a\}\leq e^{-2a^2/n}
\end{equation}

The first lemma is a non-trivial variant of the Chernoff-Hoeffding
bound, first introduced in our recent work~\cite{W.Dai:2011}, that
allows for bounded differences between the conditional expectations
of sequence of random variables that we revealed sequentially:

\lemma \label{lemma:chernoff} Let $X_1,\cdots,X_n$ be random
variables with range $[0,b]$ and such that
$|\mathbb{E}[X_t|X_1,\cdots,X_{t-1}]- \mu | \leq C$. $C$ is a
constant number such that $0< C < \mu$. Let $S_n = X_1+\cdots+X_n$.
Then for all  $a \geq 0$,
\begin{equation} \label{eqn:chernoff1}
\mathbb{P}\{S_n \geq n(\mu+C) + a\}\leq
e^{-2(\frac{a(\mu-C)}{b(\mu+C)})^2/n}
\end{equation} and
\begin{equation} \label{eqn:chernoff2}
\mathbb{P}\{S_n \leq n(\mu-C) - a\}\leq e^{-2(a/b)^2/n}
\end{equation}

\begin{proof} We first prove (\ref{eqn:chernoff1}). We generate random variables
$\hat{X}_1, \hat{X}_2,\cdots, \hat{X}_n$ as follows:

$\hat{X}_1 = (\mu+C)\frac{X_1}{\mathbb{E}[X_1]},$

$\hat{X}_2 = (\mu+C)\frac{X_2}{\mathbb{E}[X_2|\hat{X}_1]},$

$\cdots$

$\hat{X}_t =
(\mu+C)\frac{X_t}{\mathbb{E}[X_t|\hat{X}_1,\hat{X}_2,\cdots,\hat{X}_{t-1}]}.$

Note that
\begin{equation}
|\mathbb{E}[X_t|X_1,\cdots,X_{t-1}]- \mu | \leq C \nonumber
\end{equation}
 So we have
 \begin{equation}
 |\mathbb{E}[X_t|\hat{X}_1,\cdots,\hat{X}_{t-1}]- \mu | \leq C \nonumber
\end{equation}

Since $\frac{\hat{X}_t}{X_t}$ is at least 1, at most
$\frac{\mu+C}{\mu-C}$, $\hat{X}_1, \hat{X}_2,\cdots, \hat{X}_n$ have
finite support (they are in the range $[0, b\frac{\mu+C}{\mu-C}]$).
Besides, $\mathbb{E}[\hat{X}_t|\hat{X}_1,\cdots,\hat{X}_{t-1}] = \mu
+ C$, $\forall t$.

Let $\hat{S}_n = \hat{X}_1+\hat{X}_2+\cdots+\hat{X}_n$, then for all
$a \geq 0$,
\begin{equation}
\begin{split}
\mathbb{P}\{S_n \geq n(\mu+C) + a\} & \leq \mathbb{P}\{\hat{S}_n
\geq
n(\mu+C) + a\} \\
& \leq e^{-2(\frac{a(\mu-C)}{b(\mu+C)})^2/n}
\end{split}
\end{equation}
The first inequality stands because $\frac{\hat{X}_t}{X_t} \geq
1$,$\forall t$. The second inequality stands because of Fact 1.

The proof of (\ref{eqn:chernoff2}) is similar. We generate random
variables $\hat{X}_1',\hat{X}_2',\cdots,\hat{X}_n'$ as follows:

$\hat{X}_1' = (\mu-C)\frac{X_1}{\mathbb{E}[X_1]},$

$\cdots$

$\hat{X}_n' =
(\mu-C)\frac{X_n}{\mathbb{E}[X_n|\hat{X}_1',\hat{X}_2',\cdots,\hat{X}_{n-1}']}.$

Note that
\begin{equation}
|\mathbb{E}[X_t|X_1,\cdots,X_{t-1}]-\mu|\leq C \nonumber
\end{equation}
So we have
\begin{equation}
|\mathbb{E}[X_t'|\hat{X}_1',\cdots,\hat{X}_{t-1}']- \mu | \leq C
\nonumber
\end{equation}
$\frac{\hat{X}_t'}{X_t}$ is at most 1, at least
$\frac{\mu-C}{\mu+C}$, therefore $\hat{X}_1, \hat{X}_2,\cdots,
\hat{X}_n$ have finite support (they are in the range $[0,b]$).
Besides, $\mathbb{E}[\hat{X}_t'|\hat{X}_1',\cdots,\hat{X}_{t-1}'] =
\mu - C$, $\forall t$.

Let $\hat{S}_n' = \hat{X}_1'+\hat{X}_2'+\cdots+\hat{X}_n'$, then for
all $a \geq 0$,
\begin{equation}
\begin{split}
\mathbb{P}\{S_n \leq n(\mu-C) - a\} & \leq \mathbb{P}\{\hat{S}_n'
\leq
n(\mu-C) - a\}\\
& \leq e^{-2(a/b)^2/n}
\end{split}
\end{equation}
The first inequality stands because $\frac{\hat{X}_t'}{X_t} \leq
1$,$\forall t$. The second inequality stands because of Fact 1.
\end{proof}


\begin{lemma}\label{lemma:Markov} ~\cite{Anantharam:1987}
Consider an irreducible, aperiodic Markov chain with state space S,
matrix of transition probabilities P, an initial distribution
$\vec{q}$ which is positive in all states, and stationary
distribution $\vec{\pi}$($\pi_s$ is the stationary probability of
state s). The state (reward) at time $t$ is denoted by $s(t)$. Let
$\mu$ denote the mean reward. If we play the chain for an arbitrary
time T, then there exists a value $A_P \leq (\min_{s\in
S}\pi_s)^{-1}\sum_{s \in S}s$ such that
$\mathbb{E}[\sum_{t=1}^Ts(t)-\mu T]\leq A_P$.
\end{lemma}

Lemma \ref{lemma:Markov} shows that if a player keeps selecting the
optimal arm, the difference between the expected reward and the
highest stationary reward is bounded by a constant. Hence if the
player switches from the optimal arm to one another, the reward loss
caused by switching can be bounded.

Based on these two lemmas, we can give the proof of Theorem
\ref{theoremRA:K=1} show as below.

\begin{proof} Since $K=1$, $\sigma^{(1)}$ is the index of the optimal arm. The
regret comes from two parts: the regret when selecting an arm other
than arm $\sigma^{(1)}$; the difference between $\mu^{\sigma(1)}$
and $\mathbb{E}(Y^{\Phi}(t))$ when selecting arm $\sigma^{(1)}$.
From Lemma \ref{lemma:Markov}, we know that each time when we switch
from arm $\sigma^{(1)}$ to one another, at most we lose a constant
value from the second part of the regret. If the number of
selections of one arm other than $\sigma^{(1)}$ in line \ref{line:9}
is bounded by $O(\ln {n})$, the first part of regret can be bounded
by $O(G(n)\ln{n})$ and the second part can be bounded by
$A_PO(\ln{n})$, and the total regret can be bounded by
$O(G(n)\ln{n})$. So next we will show this is true.

For ease of exposition, we discuss the time slots $n$ such that
$G||n$, where $G||n$ denotes the time $n$ is the end of certain
step.

We define $q$ as the smallest index such that
\begin{equation}\label{eqn:Bq}
B_q \geq \lceil \max
\{\frac{2C_P}{\mu^{\sigma(1)}-\mu^{\sigma{(2)}}},
\frac{C_P}{\mu^{\sigma(l)}},l=1,2,\cdots,N\}\rceil
\end{equation}
 where
\begin{equation}
C_P =  \max_{1\leq i\leq N}\{(\min_{x\in S^i}\pi_x^i)^{-1}\sum_{s
\in S^i}s\} \nonumber
\end{equation}

Let
\begin{equation}
c_{t,s} = \sqrt{(L\ln{t})/s} \nonumber
\end{equation}
\begin{equation}\label{eqn:w^*}
w^* = q(\mu^{\sigma(1)}-\frac{C_P}{B_q})
\end{equation}
and
\begin{equation}\label{eqn:w^i}
w^i =
q\frac{\mu^{\sigma(i)}-C_P/B_q}{\mu^{\sigma(i)}+C_P/B_q}(\mu^{\sigma(i)}+\frac{C_P}{B_q}-1)
\end{equation}

Next we will show that it is possible to define $\alpha^*$ such that
if arm $\sigma(1)$ is selected for $s(>\alpha^*)$ steps, then
\begin{equation} \label{eqn:p1}
\exp(-2(w^*-sc_{t,s})^2/(s-q)) \leq t^{-4}.
\end{equation}

In fact, when $s > \max{\{q, \lceil w^*/(\sqrt{L}-\sqrt{2})\rceil
^2}\}$,  we have
\begin{equation}
\sqrt{Ls}-w^* \geq \sqrt{2(s-q)}\nonumber
\end{equation}
 Consider
 \begin{equation}
 f(t) = \sqrt{Ls\ln{t}}-w^*- \sqrt{2(s-q)\ln{t}}, \quad \forall t \geq e \nonumber
 \end{equation}

Since $f(t)$ is an increasing function and $f(e) \geq 0$, we have
\begin{equation}
f(t) \geq 0, \forall t \geq e \nonumber
\end{equation}

i.e. $\sqrt{Ls\ln{t}}-w^* \geq \sqrt{2(s-q)\ln{t}}$.
 And this equals to
 \begin{equation}\exp(-2(w^*-sc_{t,s})^2/(s-q))
\leq t^{-4}\nonumber
\end{equation} Thus at least we can set
\begin{equation}\label{eqn:alpha^*}
\alpha^* =1+\lceil \max{\{q, [w^*/(\sqrt{L}-\sqrt{2})]^2}\}\rceil
\end{equation}

For the similar reason, we could define
\begin{equation}\label{eqn:alpha_i}
\alpha^i =1+ \lceil \max{\{q, [w^i/(\sqrt{L}-\sqrt{2})]^2}\}\rceil
\end{equation} such that if
arm $\sigma(i)$ is selected for $s(>\alpha^i)$ steps,
\begin{equation}\label{eqn:p2}
\exp(\frac{-2(w^i+sc_{t,s})^2}{s-q})\\
\leq t^{-4}
\end{equation}

Moreover, we will show that there exists

\begin{equation}\label{eqn:gamma}
\begin{split}
\gamma &= \lceil
\max\{{(N-1)(4\alpha^*+1)+\alpha^*,(N-1)e^{4\alpha^*/L}+\alpha^*},\\
& \max_{ 2\leq i\leq
N}\{(N-1)(4\alpha^i+1)+\alpha^i,(N-1)e^{4\alpha^i/L}+\alpha^i\}\}\rceil
\end{split}
\end{equation}
such that for the time $n$, if $G(n)>B_{\gamma}$, then arm
$\sigma(1)$ is selected at least $\alpha^*$ times and arm
$\sigma(i)$ is selected at least $\alpha^i$ times.

In fact, if arm $\sigma(1)$ has been selected less than $\alpha^*$
times, consider arm $j$ being selected for the most steps. Consider
the last time selecting arm $j$, denote that time as $t$, there must
be
\begin{equation}
\frac{\hat{X}_{{\sigma(1)}}}{i_{\sigma(1)}}+c_{t,i_{\sigma(1)}}\leq
\frac{\hat{X}_j}{i_j}+c_{t,i_j} \nonumber
\end{equation}

Since arm $j$ has been selected the most times, we have $i_j \geq
\max\{4\alpha^*+1,e^{4\alpha^*/L} \}$. Noting that
$\frac{\hat{X}_{{\sigma(1)}}}{i_{\sigma(1)}} \geq 0$,
$\frac{\hat{X}_j}{i_j} \leq 1$, $i_{\sigma(1)} \leq \alpha^* -1$,
$i_j \geq 4\alpha^* + 1$,  we have
\begin{equation}
0 + \sqrt{\frac{L\ln{t}}{\alpha^*-1}} \leq 1+
\sqrt{\frac{L\ln{t}}{4\alpha^*+1}} \nonumber
\end{equation}

 Consider
 \begin{equation}
 g(t) =  1+ \sqrt{\frac{L\ln{t}}{4\alpha^*+1}}- \sqrt{\frac{L\ln{t}}{\alpha^*-1}}\nonumber
 \end{equation}

  Since $g(t)$ is a decreasing function and $t \geq \sum_{l=1}^{e^{4\alpha^*/L}} B_l  \geq e^{4\alpha^*/L}$, we have
 \begin{equation}
  g(t) \leq g(e^{4\alpha^*/L}) = 1+ \sqrt{\frac{4\alpha^*}{4\alpha^*+1}}- \sqrt{\frac{4\alpha^*}{\alpha^*-1}} < 0 \nonumber
 \end{equation}
  This contradicts the conclusion above. So arm $\sigma(1)$ has been played at least $\alpha^*$ times.

If we replace $\alpha^*$ with $\alpha^i$ and replace arm $\sigma(1)$
with arm $\sigma(i)$, without changing the proof, we can conclude
that arm $\sigma(i)$ has been played at least $\alpha^i$ times.

Next we will bound the number of times we fail to choose the optimal
arm. We will show that this number has a logarithmic order.

Denote $T_j(n)$ as the number of times we select  arm $\sigma(j)$ up
to time $n$. Then, for any positive integer $l$, we have
\begin{equation}
\begin{split} T_j(n) & =
1+\sum_{t=\sum_{i=1}^NB_i,G||t}^{n}\mathbb{I}\{\frac{\hat{X}_{\sigma(1)}(t)}{i_{\sigma(1)}(t)}+c_{t,i_{\sigma(1)}}
\\
 & \quad <
\frac{\hat{X}_{\sigma(j)}(t)}{i_{\sigma(j)}(t)}+c_{t,i_j}\} \\
& \leq l+\gamma+\\
&\sum_{t=B_1+\cdots+B_\gamma,G||t}^{n}\sum_{s_1=\alpha^*}^{\alpha(t),t=B_1+\cdots+B_{\alpha(t)}}\sum_{s_j=\max(\alpha^j,l)}^{\beta(t),t=B_1+\cdots+B_{\beta(t)}}\\
&\mathbb{I}\{\frac{\hat{X}_{\sigma(1),s_1}}{s_1}+c_{t,s_1}\leq
\frac{\hat{X}_{\sigma(j),s_j}}{s_j}+c_{t,s_j}\}
\end{split}
\end{equation}
where $\mathbb{I}\{x\}$ is the index function defined to be 1 when
the predicate $x$ is true, and 0 when it is a false predicate;
$i_{\sigma(j)}(t)$ is the number of times we select arm $\sigma(j)$
when up to time $t, \forall j=2,\cdots,N$; $\hat{X}_{\sigma(j)}(t)$
is the sum of every sample mean of arm $\sigma(j)$ for
$i_{\sigma(j)}(t)$ plays up to time $t$; $\hat{X}_{\sigma(j),s_j}$
is the sum of every sample mean for $s_j$ times selecting arm
$\sigma(j)$.

The condition $\{\frac{\hat{X}_{\sigma(1),s_1}}{s_1}+c_{t,s_1}\leq
\frac{\hat{X}_{\sigma(j),s_j}}{s_j}+c_{t,s_j}\}$ implies that at
least one of the following must hold:
\begin{equation} \label{eqn:inequ1}
\frac{\hat{X}_{\sigma(1),s_1}}{s_1}\leq
\mu^{\sigma(1)}-\frac{C_P}{B_q}-c_{t,s_1}
\end{equation}
\begin{equation} \label{eqn:inequ2}
\frac{\hat{X}_{\sigma(j),s_j}}{s_j}\geq
\mu^{\sigma(j)}+\frac{C_P}{B_q}+\frac{\mu^{\sigma(j)}+C_P/B_q}{\mu^{\sigma(j)}-C_P/B_q}c_{t,s_j}
\end{equation}
\begin{equation} \label{eqn:inequ3}
\mu^{\sigma(1)}-\frac{C_P}{B_q} <
\mu^{\sigma(j)}+\frac{C_P}{B_q}+(1+\frac{\mu^{\sigma(j)}+C_P/B_q}{\mu^{\sigma(j)}-C_P/B_q})c_{t,s_j}
\end{equation}

Note that
$\hat{X}_{\sigma(1),s_1}=\hat{A}_{\sigma(1),1}+\hat{A}_{\sigma(1),2}+\cdots+\hat{A}_{\sigma(1),s_1}$,
where $\hat{A}_{\sigma(1),i}$ is sample average reward for the
$i_{th}$ step selecting arm $\sigma(1)$. From Lemma
\ref{lemma:Markov}, we have

\begin{equation}
\mu^{\sigma(1)}-\frac{C_P}{B_q} \leq \mathbb{E}[\hat{A}_{1,i}] \leq
\mu^{\sigma(1)}+\frac{C_P}{B_q} \quad \forall i \geq q
\end{equation}

Then applying Lemma \ref{lemma:chernoff}, and the results in
(\ref{eqn:p1}) and (\ref{eqn:p2}), we have:
\begin{equation}
\begin{split}
&\mathbb{P}(\frac{\hat{X}_{\sigma(1),s_1}}{s_1}\leq
\mu^{\sigma(1)}-\frac{C_P}{B_q}-c_{t,s_1})\\
&=
\mathbb{P}(\frac{\hat{A}_{\sigma(1),1}+\cdots+\hat{A}_{\sigma(1),s_1}}{s_1}\leq
\mu^{\sigma(1)}-\frac{C_P}{B_q}-c_{t,s_1})\\
& \leq
\mathbb{P}(\frac{0+\cdots+0+\hat{A}_{\sigma(1),q+1}+\cdots+\hat{A}_{\sigma(1),s_1}}{s_1}
\leq
\mu^{\sigma(1)}\\
&-\frac{C_P}{B_q}-c_{t,s_1})\\
& \leq \exp(-2(w^*-sc_{t,s_1})^2/(s_1-q)) \leq t^{-4}
\end{split}
\end{equation}

\begin{equation}
\begin{split}
&\mathbb{P}(\frac{\hat{X}_{\sigma(j),s_j}}{s_j}\geq
\mu^{\sigma(j)}+\frac{C_P}{B_q}+\frac{\mu^{\sigma(j)}+C_P/B_q}{\mu^{\sigma(j)}-C_P/B_q}c_{t,s_j})\\
& =
\mathbb{P}(\frac{\hat{A}_{\sigma(j),1}+\cdots+\hat{A}_{\sigma(j),s_j}}{s_j}\geq
\mu^{\sigma(j)}+\frac{C_P}{B_q}\\
&+\frac{\mu^{\sigma(j)}+C_P/B_q}{\mu^{\sigma(j)}-C_P/B_q}c_{t,s_j})\\
&\leq
\mathbb{P}(\frac{1+\cdots+1+\hat{A}_{\sigma(j),q+1}+\hat{A}_{\sigma(j),s_j}}{s_j}\geq
\mu^{\sigma(j)}+\frac{C_P}{B_q}\\
&+\frac{\mu^{\sigma(j)}+C_P/B_q}{\mu^{\sigma(j)}-C_P/B_q}c_{t,s_j}) \\
&\leq \exp(\frac{-2(w^j+sc_{t,s_j})^2}{s_j-q})\leq t^{-4}
\end{split}
\end{equation}

Denote $\lambda_j(n)$ as
\begin{equation}
\begin{split}
&\lambda_j(n)
=\lceil(L(1+\frac{\mu^{\sigma(j)}+C_P/B_q}{\mu^{\sigma(j)}-C_P/B_q})^2\ln{n})/(\mu^{\sigma(1)}-\mu^{\sigma(j)}\\
&-\frac{2C_P}{B_q})^2\rceil
\end{split}
\end{equation}
For $l \geq \lambda_j(n)$, (\ref{eqn:inequ3}) is false. So we get:
\begin{equation} \label{ineqn:boundT}
\begin{split}
&\mathbb{E}(T_j(n)) \leq \lambda_j(n) +\gamma+
\Sigma_{t=1}^{\infty}\Sigma_{s_1=1}^{t}\Sigma_{s_j=1}^{t}2t^{-4}\\
&\leq \lambda_j(n)+\gamma+\frac{\pi^2}{3}.
\end{split}
\end{equation}

As we analysis before, the first part of the regret is bounded by
\begin{equation}\sum_{j=2}^N\mathbb{E}[T_j(n)](G(n)(\mu^{\sigma(1)}-\mu^{\sigma(j)})+2C_P)\nonumber
\end{equation} and the second part is bounded by $C_P\sum_{j=2}^N\mathbb{E}(T_j(n)$.

Therefore, we have:
\begin{equation}
\begin{split}
& r^{\Phi}(n) \leq G(n)+ \\
& \quad\quad \sum_{j=2}^N
(G(n)(\mu^{\sigma(1)}-\mu^{\sigma(j)})+3C_P)
(\lambda_j(n)+\gamma+\frac{\pi^2}{3})
\end{split}
\end{equation}

This inequality can be readily translated to the simplified form of
the bound given in the statement of Theorem 1, where:
\begin{equation}
\begin{split}
&Z_1 = \sum_{j=2}^N (\mu^{\sigma(1)}-\mu^{\sigma(j)})\lceil \frac{L(1+\frac{\mu^{\sigma(j)}+C_P/B_q}{\mu^{\sigma(j)}-C_P/B_q})^2}{(\mu^{\sigma(1)}-\mu^{\sigma(j)}-\frac{2C_P}{B_q})^2}\rceil\\
& Z_2 = 3C_P\sum_{j=2}^N\lceil \frac{L(1+\frac{\mu^{\sigma(j)}+C_P/B_q}{\mu^{\sigma(j)}-C_P/B_q})^2}{(\mu^{\sigma(1)}-\mu^{\sigma(j)}-\frac{2C_P}{B_q})^2}\rceil\\
& Z_3 = (\gamma+\frac{\pi^2}{3})\sum_{j=2}^N
(\mu^{\sigma(1)}-\mu^{\sigma(j)})+1\\ \nonumber & Z_4 =
3(N-1)C_P(\gamma+\frac{\pi^2}{3})
\end{split}
\end{equation}

\end{proof}

\subsection{Corollary}
From the analysis above, we see that if sequence
$\{B_i\}_{i=1}^{\infty}$ is constant and $B_i \geq \lceil \max
\{\frac{2C_P}{\mu^{\sigma(1)}-\mu^{\sigma{(2)}}},
\frac{C_P}{\mu^{\sigma(l)}},l=1,2,\cdots,N\}\rceil$, then Algorithm
1 achieves logarithmic regret over time. Specifically, we have the
following corollary:

\corollary The system model is the same as that in Theorem
\ref{theoremRA:K=1}. In Algorithm 1, if
\begin{equation}
B_i \equiv \lceil \max
\{\frac{2C_P}{\mu^{\sigma(1)}-\mu^{\sigma{(2)}}},
\frac{C_P}{\mu^{\sigma(l)}},l=1,2,\cdots,N\}\rceil \forall i \in
\mathds{N}\nonumber
\end{equation}
then the expected regret after n time slots is at most $Z_1'
B_1\ln{n} + Z_2' \ln{n} + Z_3' B_1 + Z_4'$, where
\begin{equation}
\begin{split}
&Z_1' = \sum_{j=2}^N (\mu^{\sigma(1)}-\mu^{\sigma(j)})\lceil \frac{L(1+\frac{\mu^{\sigma(j)}+C_P/B_1}{\mu^{\sigma(j)}-C_P/B_1})^2}{(\mu^{\sigma(1)}-\mu^{\sigma(j)}-\frac{2C_P}{B_1})^2}\rceil\\
& Z_2' = 3C_P\sum_{j=2}^N\lceil \frac{L(1+\frac{\mu^{\sigma(j)}+C_P/B_1}{\mu^{\sigma(j)}-C_P/B_1})^2}{(\mu^{\sigma(1)}-\mu^{\sigma(j)}-\frac{2C_P}{B_1})^2}\rceil\\
& Z_3' = (\gamma_1+\frac{\pi^2}{3})\sum_{j=2}^N
(\mu^{\sigma(1)}-\mu^{\sigma(j)})+1\\ \nonumber & Z_4' =
3(N-1)C_P(\gamma_1+\frac{\pi^2}{3})
\end{split}
\end{equation}
 and
here $\gamma_1$ is obtained given $q=1$ in
(\ref{eqn:alpha^*}), (\ref{eqn:alpha_i}), (\ref{eqn:w^*}), (\ref{eqn:w^i})
and (\ref{eqn:gamma}).

\textbf{Remark}: This corollary is just a special case for Theorem
\ref{theoremRA:K=1}, but it reveals the fact that when certain
knowledge of the system is available (in this case, some bounds
related to the stationary state distribution and state-dependent
rewards), we can design an algorithm that achieves logarithmic
regret over time.

\section{Analysis for Multi-Arm Selection}
\label{sec:Sensing_K} In this section, we discuss the general case
where $K$ is a known positive integer. We show a generalization of
the CEE algorithm and prove that it still achieves a
near-logarithmic regret with time.

\subsection{Algorithm Design}
The basic idea is similar to Algorithm 1: first initialize and then
find  the optimal indices. The only difference is here we have to
select $K$ indices that obtain the greatest value in line
\ref{line:9} at one time. The definition of $\{B_i\}_{i=1}^\infty$
stays the same and the details are shown in in Algorithm
\ref{alg:sensing2}.

\begin{algorithm} [ht]
\caption{Continuous Exploration and Exploitation (CEE): Multi-Arm
Selection} \label{alg:sensing2}
\begin{algorithmic}[1]
\State \commnt{ Initialization}

\State Sequently play $K$ arms $B_i$ times until every arm is
selected once, $i=1,2,\cdots,\lceil \frac{N}{K}\rceil$. Denote
$\hat{A}_j$ as the sample mean of the corresponding  $B_i$ rewards
of arm $j$ , $i=1,2,\cdots,\lceil \frac{N}{K}\rceil$,
$j=1,2,\cdots,N$ \State $\hat{X}_i = \hat{A}_i$,$i=1,2,\cdots,N$
\State $n = \sum_{i=1}^{\lceil \frac{N}{K}\rceil} B_i$ \State
$i=\lceil \frac{N}{K}\rceil+1$, $i_j=1$, $j=1,2,\cdots,N$

\State \commnt{Main loop}

\While {1}
    \State Denote $F(j) = \frac{\hat{X}_j}{i_j}+\sqrt{\frac{L\ln{n}}{i_j}}$( L can be any constant larger than 2)
    \State Find arm $j_1,j_2,\cdots,j_K$ such that
    \begin{equation}
    \begin{split}
    &F(j_1) \geq F(j_2)\geq \cdots \geq F(j_K)\geq F(l) \nonumber\\
    &\forall l  \notin \{j_1,j_2,\cdots,j_K\}
    \end{split}
    \end{equation}

    \State $i_{j_l} = i_{j_l}+1, 1\leq l\leq K$
    \State Select arm $j_1,j_2,\cdots,j_K$ and play for $B_i$ times, let $\hat{A}_{j_l}(i_{j_l})$ record the sample mean of these $B_i$ rewards
    \State $\hat{X}_{j_l} = \hat{X}_{j_l} + \hat{A}_{j_l}(i_{j_l})$
    \State $i = i+1$
    \State $n = n + B_i$;
\EndWhile
\end{algorithmic}
\end{algorithm}

\subsection{Regret Analysis}
In this subsection, we keep the definition of $G(n)$ in
(\ref{eq:G(n)}) and the definition of $regret$ in
(\ref{eqn:regret}). We will show that the regret achieved by
Algorithm \ref{alg:sensing2} has a near logarithmic order. This is
given in the following Theorem \ref{theoremRA:K}.

\begin{theorem}\label{theoremRA:K}
Assume all arms are modeled as finite state, irreducible, aperiodic
and reversible Markov chains. All the states (rewards) are positive.
The expected regret with Algorithm \ref{alg:sensing2} after $n$ time
steps is at most $Z_5 G(n)\ln{n} + Z_6 \ln{n} + Z_7 G(n) + Z_8$,
where $Z_5,Z_6,Z_7,Z_8$ are constants only related to $P_i,
i=1,2,\cdots, N$, explicit expressions are at the end of proof for
Theorem \ref{theoremRA:K}.

\end{theorem}
\begin{proof} The proof of Theorem \ref{theoremRA:K} is similar to that of Theorem \ref{theoremRA:K=1}. We still divide the regret into two parts and bound them separately. We keep the denotation of $G||n$ and discuss the time slots such that $G||n$.

We define $q'$ as the smallest index such that
\begin{equation}
B_{q'} \geq \lceil
\max\{\frac{2C_P}{\mu^{\sigma(K)}-\mu^{\sigma{(K+1)}}},\frac{C_P}{\mu^{\sigma(l)}},l=1,2,\cdots,N\}\rceil
\end{equation}

Let
\begin{equation}\label{eqn:m_j}
m_j^*=q'(\mu^{\sigma(j)}-\frac{C_P}{B_{q'}}), 1\leq j\leq K
\end{equation} and
\begin{equation}\label{eqn:m_i^*}
m^i=q'\frac{\mu^{\sigma(i)}-C_P/B_{q'}}{\mu^{\sigma(i)}+C_P/B_{q'}}(\mu^{\sigma(i)}+C_P/B_{q'}-1)
,K+1\leq i\leq N
\end{equation}

As shown in the proof of Theorem \ref{theoremRA:K=1}, if we set
\begin{equation}\label{eqn:beta^*}
\beta_j^* = 1+\lceil \max\{q',[m_j^*/(\sqrt{L}-\sqrt{2})]^2\}\rceil
,1\leq j\leq K
\end{equation}
\begin{equation}\label{eqn:beta^i}
\beta^i = 1+\lceil \max\{q',[m^i/(\sqrt{L}-\sqrt{2})]^2\}\rceil
,K+1\leq i\leq N
\end{equation}
and if $s>\beta_j^*$ and $s>\beta^i$ we will have
\begin{equation} \label{eqn:p3}
\exp(\frac{-2(m_j^*-sc_{t,s})^2}{s-q'}) \leq t^{-4}.
\end{equation}
and
\begin{equation}\label{eqn:p4}
\exp(\frac{-2(m^j+sc_{t,s})^2}{s-q'})\\
\leq t^{-4}.
\end{equation}

Moreover, we will show that there exists

\begin{equation}\label{eqn:gamma'}
\begin{split}
\gamma' &= \lceil\max(
\max_{1\leq j\leq K}\{(N-1)(5\beta_j^*+1)+\beta_j^*,(N-1)(e^{4\beta_j^*/L}\\
& +\beta_j^*)+\beta_j^*\}, \max_{ K+1\leq i\leq N}\{(N-1)(5\beta^i+1)+\beta^i,(N-\\
&1)(e^{4\beta^i/L}+\beta^i)+\beta^i\})\rceil
\end{split}
\end{equation}
such that for the time $n$, if $G(n)>B_{\gamma'}$, then arm
$\sigma(j)$ is played at least $\beta_j^*$ times and arm $\sigma(i)$
is played at least $\beta^i$ times, where $1\leq j\leq K, K+1\leq
i\leq N$.

In fact, if arm $\sigma(j)$ has been played less than $\beta_j^*$
times, then there exist an arm $\sigma(l)(K+1 \leq l\leq N)$ that
has been played the most times. Consider the last time that arm
$\sigma(l)$ is selected and arm $\sigma(j)$ is not selected, and
denote that time as $t$; Then it must be true that
\begin{equation}
\frac{\hat{X}_{{\sigma(j)}}}{i_{\sigma(j)}}+c_{t,i_{\sigma(j)}}\leq
\frac{\hat{X}_{\sigma(l)}}{i_{\sigma(l)}}+c_{t,i_{\sigma(l)}}
\nonumber
\end{equation}

Since arm $\sigma(l)$ has been played the most times, we have
$i_{\sigma(l)} \geq \max\{4\beta_j^*+1,e^{4\beta_j^*/L} \}$. Noting
that $\frac{\hat{X}_{{\sigma(j)}}}{i_{\sigma(j)}} \geq 0$,
$\frac{\hat{X}_{{\sigma(l)}}}{i_{\sigma(l)}} \leq 1$, $i_{\sigma(j)}
\leq \beta_j^* -1$,$i_{\sigma(l)} \geq 4\beta_j^* + 1$,  we have
\begin{equation}
0 + \sqrt{\frac{L\ln{t}}{\beta_j^*-1}} \leq 1+
\sqrt{\frac{L\ln{t}}{4\beta_j^*+1}} \nonumber
\end{equation}

 Consider
 \begin{equation}
 g^*(t) =  1+ \sqrt{\frac{L\ln{t}}{4\beta_j^*+1}}- \sqrt{\frac{L\ln{t}}{\beta_j^*-1}}\nonumber
 \end{equation}

  Since $g*(t)$ is a decreasing function and $t \geq \sum_{l=1}^{e^{4\beta_j^*/L}} B_l  \geq e^{4\beta_j^*/L}$, we have
 \begin{equation}
  g^*(t) \leq g^*(e^{4\beta_j^*/L}) = 1+ \sqrt{\frac{4\beta_j^*}{4\beta_j^*+1}}- \sqrt{\frac{4\beta_j^*}{\beta_j^*-1}} < 0 \nonumber
 \end{equation}

  This contradicts the conclusion above. So arm $\sigma(j)$ has been played at least $\beta_j^*$ times.

If we replace $\beta_j^*$ with $\beta^i$ and replace arm $\sigma(j)$
with arm $\sigma(i)$, without changing the proof, we can conclude
that arm $\sigma(i)$ has been played at least $\beta^i$ times, $
K+1\leq i\leq N$.

Based on the conclusions above, we can bound the expectation of the
number of non-optimal arm choices. We keep the denotation of
$T_j(n)$ and $\mathbb{I}\{x\}$ except that here $K+1 \leq j \leq N$.
Every time we select $\sigma(j)$, there must exist an arm from
$\sigma(1)$ to $\sigma(K)$ not being chosen. We denote that unknown
arm as  $\sigma(r,t)$(if more than one arm not chosen, pick any of
them).
\begin{equation}
\begin{split} T_j(n) & =
1+\sum_{t=\sum_{i=1}^NB_i,G||t}^{n}\mathbb{I}\{\frac{\hat{X}_{\sigma(r,t)}(t)}{i_{\sigma(r,t)}(t)}+c_{t,i_{\sigma(r,t)}}< \\
&\frac{\hat{X}_{\sigma(j)}(t)}{i_{\sigma(j)}(t)}+c_{t,i_j}\}
\end{split}
\end{equation}


And if we replace $\sigma(1)$ with $\sigma(r,t)$, according to the
deduction from (\ref{eqn:inequ1}) to (\ref{ineqn:boundT}), we
conclude that
\begin{equation}
\begin{split}
&\mathbb{E}(T_j(n)) \leq 1+
\max_{1\leq i \leq K}(\lambda_{i,j}(n) +\gamma'+\frac{\pi^2}{3})\\
& = 1+ \lambda_{K,j}(n) +\gamma'+\frac{\pi^2}{3}
\end{split}
\end{equation}

where
\begin{equation}
\begin{split}
&\lambda_{i,j}(n) = \lceil L(1+\frac{\mu^{\sigma(j)}+C_P/B_{q'}}{\mu^{\sigma(j)}-C_P/B_{q'}})^2\ln{n}/(\mu^{\sigma(i)}-\mu^{\sigma(j)}\\
&-\frac{2C_P}{B_{q'}})^2 \rceil    \nonumber
\end{split}
\end{equation}

Therefore, we have:
\begin{equation}
\begin{split}
& r^{\Phi}(n) \leq KG(n)+
\sum_{j=K+1}^N  (G(n)(\mu^{\sigma(1)}-\mu^{\sigma(j)})+\\
&3C_P)(\lambda_{K,j}(n)+\gamma' +\frac{\pi^2}{3})
\end{split}
\end{equation}

Equivalently, we have the simplified form of the bound given in the
statement of Theorem 2, where:
\begin{equation}
\begin{split}
&Z_5 = \sum_{j=K+1}^N (\mu^{\sigma(1)}-\mu^{\sigma(j)})\lceil L(1+\frac{\mu^{\sigma(j)}+C_P/B_{q'}}{\mu^{\sigma(j)}-C_P/B_{q'}})^2/(\mu^{\sigma(K)}\\
&-\mu^{\sigma(j)}-\frac{2C_P}{B_{q'}})^2 \rceil\\
& Z_6 = 3C_P\sum_{j=K+1}^N \lceil \frac{L(1+\frac{\mu^{\sigma(j)}+C_P/B_{q'}}{\mu^{\sigma(j)}-C_P/B_{q'}})^2}{(\mu^{\sigma(K)}-\mu^{\sigma(j)}-\frac{2C_P}{B_{q'}})^2}\rceil\\
& Z_7 = (\gamma'+\frac{\pi^2}{3})\sum_{j=K+1}^N
(\mu^{\sigma(K)}-\mu^{\sigma(j)})+K\\ \nonumber & Z_8 =
3(N-K)C_P(\gamma'+\frac{\pi^2}{3})
\end{split}
\end{equation}

\end{proof}

\subsection{Corollary}
Similarly to Section \ref{sec:Sensing}, when stationary distribution
and rewards are available, $B_i$ in Algorithm \ref{alg:sensing2} can
be a constant sequence. In this way, Algorithm \ref{alg:sensing2}
achieves arbitrarily logarithmic regret over time. Specifically, we
have Corollary 2 as follows:

\corollary The system model is the same as that in Theorem
\ref{theoremRA:K}. In Algorithm \ref{alg:sensing2}, if
\begin{equation}
\begin{split}
B_i \equiv \lceil
\max\{\frac{2C_P}{\mu^{\sigma(K)}-\mu^{\sigma{(K+1)}}},\frac{C_P}{\mu^{\sigma(l)}},l=1,2,\cdots,N\}\rceil
&\\\forall i \in \mathds{N} \nonumber
\end{split}
\end{equation}
then the expected regret after n time slots is at most $Z_5'
B_1\ln{n} + Z_6' \ln{n} + Z_7' B_1 + Z_8'$, where
\begin{equation}
\begin{split}
&Z_5' = \sum_{j=K+1}^N (\mu^{\sigma(1)}-\mu^{\sigma(j)})\lceil L(1+\frac{\mu^{\sigma(j)}+C_P/B_1}{\mu^{\sigma(j)}-C_P/B_1})^2/(\mu^{\sigma(K)}\\
&-\mu^{\sigma(j)}-\frac{2C_P}{B_1})^2 \rceil\\
& Z_6' = 3C_P\sum_{j=K+1}^N \lceil \frac{L(1+\frac{\mu^{\sigma(j)}+C_P/B_1}{\mu^{\sigma(j)}-C_P/B_1})^2}{(\mu^{\sigma(K)}-\mu^{\sigma(j)}-\frac{2C_P}{B_1})^2}\rceil\\
& Z_7' = (\gamma_2+\frac{\pi^2}{3})\sum_{j=K+1}^N
(\mu^{\sigma(K)}-\mu^{\sigma(j)})+K\\ \nonumber & Z_8' =
3(N-K)C_P(\gamma_2+\frac{\pi^2}{3})
\end{split}
\end{equation}
 and
here $\gamma_2$ is obtained given $q'=1$ in
(\ref{eqn:m_j}), (\ref{eqn:beta^*}), (\ref{eqn:beta^i}), (\ref{eqn:gamma'})
and (\ref{eqn:m_i^*}).

\section{Numerical Results}\label{sec:NR}
In this section, we simulate our algorithm  and compare it with two
previously proposed policies for this problem in the context of
opportunistic spectrum access: (1) RCA proposed by Cem Tekin
\emph{et al.}~\cite{TekinLiu} and (2) RUCB proposed by H. Liu
\emph{et al.}~\cite{Haoyang:2011}~\cite{Haoyang:2011_1}. We focus on
two properties of the algorithms: regret and variance, which show
the efficiency and stability of the algorithms respectively.

\subsection{Channel Model and Parameters}
The arms are channels. The channel model is the commonly used
Gilbert-Elliot model. The state of each channel evolves as an
irreducible, aperiodic Markov chain. Each channel has two states,
good and bad. We consider $N=5$ channels. At each time slot, the
player activates 1 channel(i.e. $K=1$). The active and passive
transition matrix for each channel are the same, i.e. $P_j=Q_j,
1\leq j\leq N$. For the ease of comparison, we set the
non-decreasing sequence $\{B_i\}_{i=1}^{\infty}$ in Algorithm 1 a
constant sequence.

We simulate three algorithms under scenario S. The transition
probabilities and rewards for this scenario are shown in
table~\ref{tab:transition}.

\begin{table}
\centering
\begin{tabular}{c|c|c}
\hline
S & $p_{01},p_{10}$ & $r_0,r_1$   \\
\hline
ch.1&0.3, 0.9 & 0.1,1 \\
\hline
ch.2&0.8, 0.7 & 0.1,1 \\
\hline
ch.3&0.5, 0.1& 0.1,1 \\
\hline
ch.4&0.2, 0.4 & 0.1,1 \\
\hline
ch.5&0.1, 0.5 & 0.1,1 \\
\hline
\end{tabular}
\vspace{0.2cm} \caption{Transition Probabilities and Rewards for
Scenario S\label{tab:transition}}
\end{table}

Intuitively, in RCA and RUCB, the regret grows with $L$. In our
algorithm, the regret grows with both $L$ and $B_i$. For fairness of
comparison, we set these parameters for all three algorithms to be
just passing the theoretical bound. In RCA~\cite{TekinLiu}, the
regret has a logarithmic order for $L \geq
112S_{\max}^2r_{\max}^2\hat{\pi}_{\max}^2/\epsilon_{\min}$, where
$S_{\max}=\max_{1\leq i \leq N}|S^i|$, $r_{\max}=\max_{x\in
S^i,1\leq i\leq N}r_x^i$, $\hat{\pi}_{\max}=\max_{x\in S^i,1\leq
i\leq N}\{\pi_x^i,1-\pi_x^i\}, \epsilon_{\min}=\min_{1\leq i \leq
K}\epsilon^i$ and $\epsilon^i$ is the eigenvalue gap of the
multiplicative symmetrization of the transition probability matrix
of the $i$th arm. In the scenario we set,
$112S_{\max}^2r_{\max}^2\hat{\pi}_{\max}^2/\epsilon_{\min}$ is
414.8148. We set $L$ 415 in RCA. In CEE Algorithm
\label{alg:sensing1}, we prove that if $B_i$ meets the requirement
stated in (\ref{eqn:Bq}) and $L>2$, the regret has a logarithmic
upper bound over time. In scenario S, the lower bound in
(\ref{eqn:Bq}) is 48.89. We set $L$ 2.1 and $B_i$ therefore to 49.
In the RUCB algorithm~\cite{Haoyang:2011}, it is required that
$L\geq
\frac{1}{\epsilon^*}(4\frac{20r_{\max}^2S_{\max}^2}{3-2\sqrt{2}}+10r_{\max}^2)$
and $D\geq \frac{4L}{(\mu^{\sigma(1)}-\mu^{\sigma(K+1)})^2}$. The
lower bounds are 3125.2 and 171480 and we accordingly set $L = 3126$
and $D = 171520$ in RUCB.

We simulate RCA, CEE and RUCB over 10 runs to calculate the regret.
The time horizon is 100 million. We also show the first 8 million
time slots of regret to compare the converging speed between RCA and
CEE. In order to access the stability of each algorithm, we also
present the variances of rewards over 100 runs for RCA, CEE and
RUCB.

The regret performance for all three algorithms are shown in Figure
\ref{fig:2} and Figure \ref{fig:4}. The reward variance for all
three algorithms is shown in Figure \ref{fig:3}.

\begin{figure*}[t]
\centering \subfigure[Regret/$\ln{\rm time}$  for RCA, CEE and RUCB]
{
\includegraphics[width=0.31\textwidth]{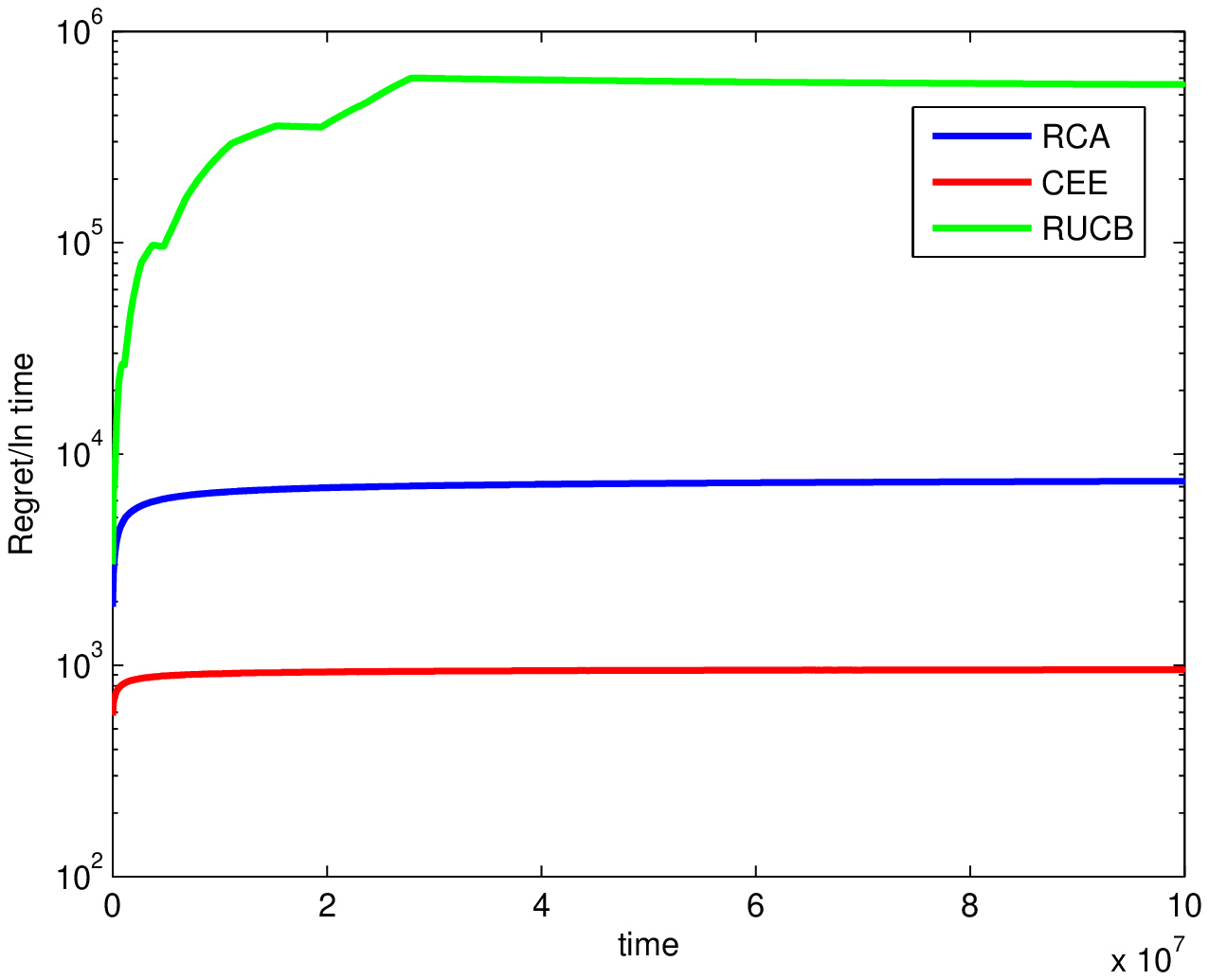}
\label{fig:2} } \subfigure[Regret/$\ln{\rm time}$ for RCA and CEE] {
\includegraphics[width=0.31\textwidth]{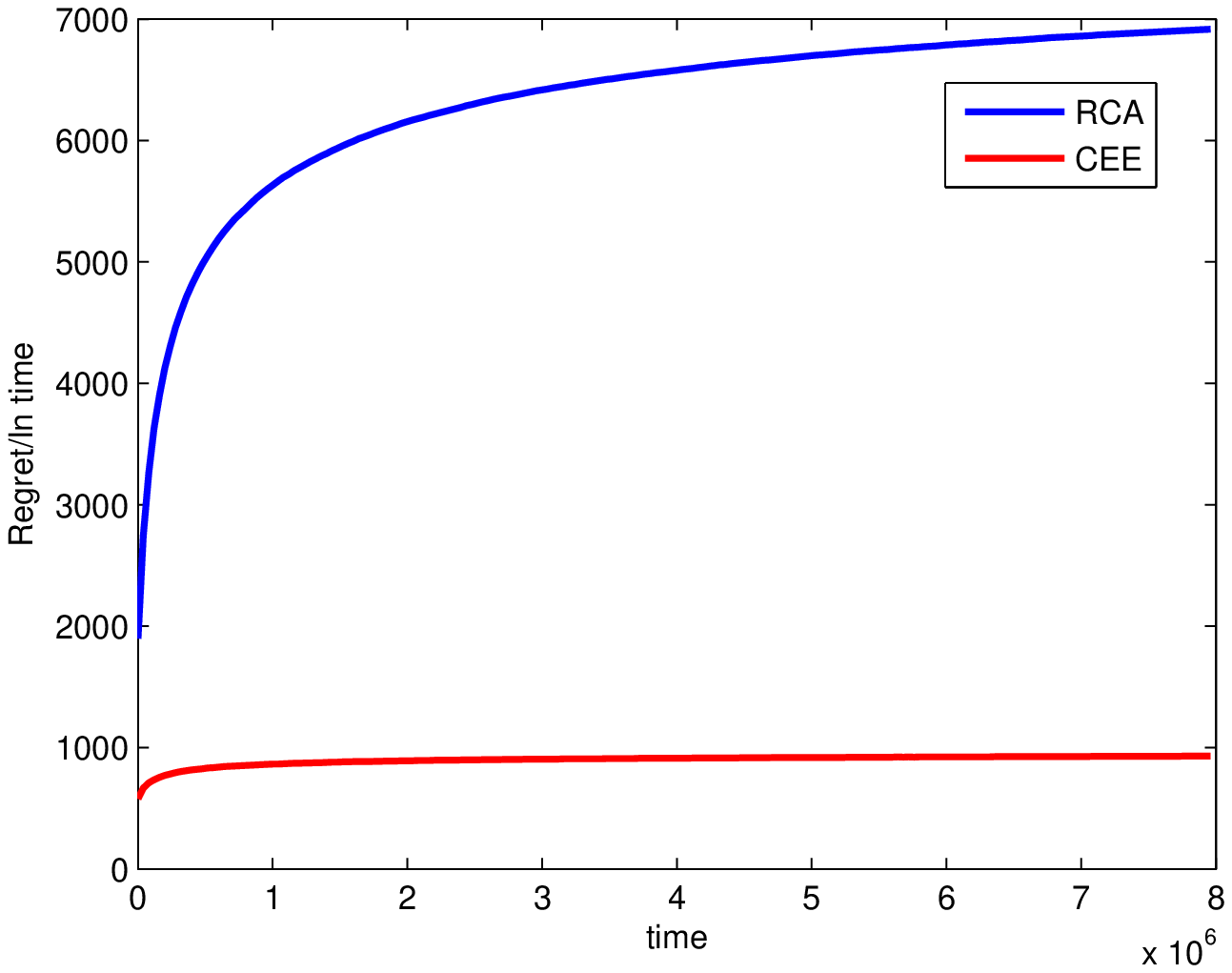}
\label{fig:4} } \subfigure[Reward variance for RCA, CEE and RUCB] {
\includegraphics[width=0.31\textwidth]{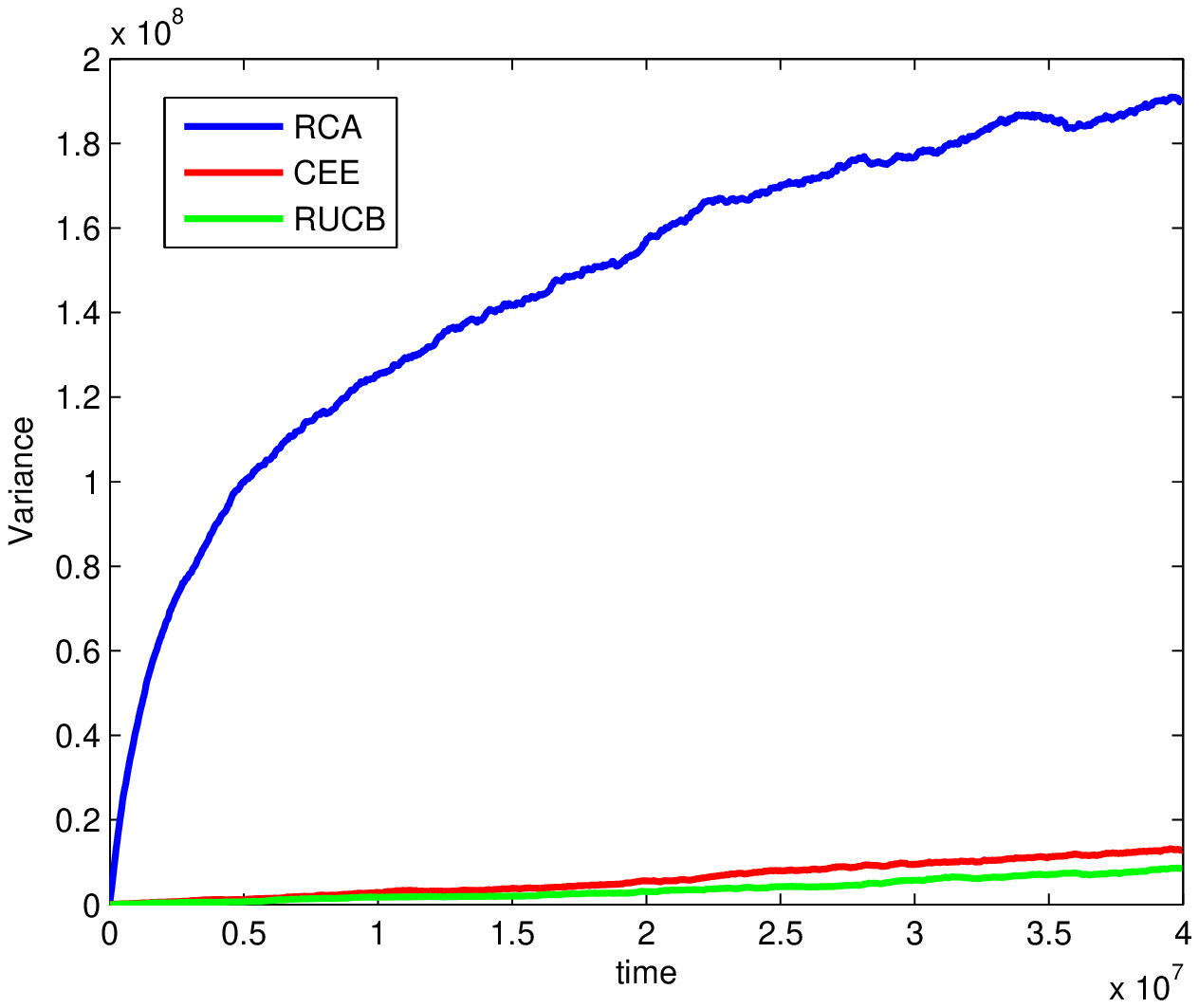}
\label{fig:3} } \caption{Regret and variance performance for RCA,
CEE and RUCB} \label{fig:matching}
\end{figure*}

\subsection{Discussion}

First of all, we note from the figures that CEE shows substantially
better regret performance than both RCA and RUCB. This is because in
CEE, the selection of arm depends on the whole observing history,
i.e. we exploit observing data in every time slot. In RCA, however,
the player chooses the arm only based on data in the second part of
each block (sub-block 2, SB2). In this way, CEE uses data much more
efficiently and the data sample means are much closer to their
expectations. As for RUCB, in exploration epoch, the player selects
every arm for certain times thus greatly reducing the chances to
play the optimal arm. It also shows the advantage of continuous
exploration and exploitation, which greatly cuts down the cost of
observing and exploring.

The second observation is that $\rm regret/\ln{\rm time}$ converges
much more quickly in CEE than in RCA and RUCB. One reason is the
regret in RCA is much greater than in Algorithm 1 so it needs more
time to reach the stationary point. Besides, as stated before, RCA
exploits data less efficiently, as the sample means are based on
only part of the observing history so they converge to the expected
value much more slowly. As for RUCB, the parameter $D$ is
considerably large and it needs quite a long time for the length of
exploration epoch to grow so that an exploitation epoch can appear.
The speed of RUCB is the slowest among these three algorithms.

Lastly, we see that the performance of RCA are much more random than
that in CEE and RUCB. The reward variances of RCA are much higher
than CEE and RUCB. The reason is that the number of time slots
between two selection in RCA is a random variable. The player stays
in the same arm until a pre-specified state is observed. In
different cases, the length of every block may vary a lot. In CEE,
however, the length of step is a constant number which greatly
reduces the randomness. In RUCB, the length of each epoch is also a
deterministic number. Besides, RUCB makes much less choices than CEE
and RCA. For these two reasons, RUCB also maintains a high
stability, albeit with poor regret performance.

In conclusion, CEE outperforms RCA and RUCB in two aspects, regret,
and convergence speed. The reward variances of RUCB and CEE
 are nearly the same, and much lower than RCA. Finally, we should note that because the boundary of parameter $B_i$ in (\ref{eqn:Bq}) is much smaller than that of parameter $L$ in RCA and $L$ and $D$ in RUCB, if we modify RCA and RUCB to make them a non-Baysian algorithm, our algorithm will converge much faster.

\section{Conclusion}
\label{sec:conclusion} In this paper, we have considered the
non-Bayesian restless multi-arm bandit problem which has been shown
to be of fundamental significance for opportunistic spectrum access
in cognitive radio networks. We use a weak notion of regret, defined
as the gap of expected reward compared to a genie who always plays
the $K$ best arms. We propose an algorithm which achieves a
near-logarithmic regret over time when no \emph{a prior} information about the system is
available. We also present another policy to achieve exact
logarithmic regret when some bounds pertaining to the stationary
state distribution and corresponding rewards are known. Compared
with prior work, this algorithm requires the least information. We
have also presented numerical results and analysis that show that
CEE significantly outperforms both of the two previously prosed
algorithms for this problem, RCA~\cite{TekinLiu} and RUCB
~\cite{Haoyang:2011}, in terms of regret and convergence speed, and
RCA in terms of reward variance.






%

\end{document}